%% file: main.tex
\documentclass{article}

\usepackage[american]{babel}

\usepackage{natbib} 
\setcitestyle{authoryear, open={(},close={)}}
\usepackage{mathtools} 
\usepackage{booktabs} 

\input{math_commands}
\usepackage{tabularx,xcolor,graphicx,booktabs}
\usepackage[linesnumbered,ruled,noend]{algorithm2e}
\usepackage{hyperref}
\definecolor{darkblue}{rgb}{0,0.08,0.45}
\hypersetup{colorlinks=true,citecolor=darkblue,linkcolor=brown,pagebackref=true,urlcolor=black,bookmarks=true}

\usepackage[accepted]{icml2021}

\icmltitlerunning{Coordinate-wise Control Variates for Deep Policy Gradients}

\begin{document}

\twocolumn[
\icmltitle{Coordinate-wise Control Variates for Deep Policy Gradients}

\icmlsetsymbol{equal}{*}

\begin{icmlauthorlist}
\icmlauthor{Yuanyi Zhong}{to}
\icmlauthor{Yuan Zhou}{to,go}
\icmlauthor{Jian Peng}{to}
\end{icmlauthorlist}

\icmlaffiliation{to}{Department of Computer Science, UIUC}
\icmlaffiliation{go}{Department of Industrial and Enterprise Systems Engineering, UIUC}

\icmlcorrespondingauthor{Yuanyi Zhong}{yuanyiz2@illinois.edu}

\icmlkeywords{Machine Learning, ICML, Reinforcement Learning, Policy Gradients}

\vskip 0.3in
]

\printAffiliationsAndNotice{}  

\thispagestyle{plain}
\pagestyle{plain}

\begin{abstract}
  The control variates (CV) method is widely used in policy gradient estimation to reduce the variance of the gradient estimators in practice. A control variate is applied by subtracting a baseline function from the state-action value estimates. Then the variance-reduced policy gradient presumably leads to higher learning efficiency. Recent research on control variates with deep neural net policies mainly focuses on scalar-valued baseline functions. The effect of vector-valued baselines is under-explored. This paper investigates variance reduction with coordinate-wise and layer-wise control variates constructed from vector-valued baselines for neural net policies. We present experimental evidence suggesting that lower variance can be obtained with such baselines than with the conventional scalar-valued baseline. We demonstrate how to equip the popular Proximal Policy Optimization (PPO) algorithm with these new control variates. We show that the resulting algorithm with proper regularization can achieve higher sample efficiency than scalar control variates in continuous control benchmarks.
\end{abstract}

\input{sec_intro}
\input{sec_related}
\input{sec_method}
\input{sec_exp}

\section{Conclusion}
We propose the coordinate-wise control variates for variance reduction in deep policy gradient methods. We demonstrate that more variance reduction can be obtained with such CVs compared to the conventional scalar CVs in a case study. We also show that this new family of CVs improves the sample efficiency of PPO in continuous control benchmarks. In the future, we hope to combine the technique with other CVs, choose the hyper-parameters automatically, understand the properties of coordinate CV better and benchmark in real-world robotic problems.

\bibliography{main}

\appendix

\onecolumn
\input{supp}

\end{document}

%% file: math_commands.tex
\usepackage{amsmath,amsfonts,bm}

\def\eqref#1{equation~\ref{#1}}

\def\1{\bm{1}}

\def\eps{{\epsilon}}

\def\rva{{\mathbf{a}}}

\def\rvg{{\mathbf{g}}}

\def\rvs{{\mathbf{s}}}

\DeclareMathAlphabet{\mathsfit}{\encodingdefault}{\sfdefault}{m}{sl}
\SetMathAlphabet{\mathsfit}{bold}{\encodingdefault}{\sfdefault}{bx}{n}

\def\gA{{\mathsf{A}}}

\def\gF{{\mathsf{F}}}

\def\gP{{\mathcal{P}}}

\def\gS{{\mathsf{S}}}

\def\sR{{\mathbb{R}}}

\def\sV{{\mathbb{V}}}

\newcommand{\E}{\mathbb{E}}

\newcommand{\Var}{\mathrm{Var}}

\DeclareMathOperator*{\argmin}{arg\,min}

\DeclareMathOperator{\Tr}{Tr}

\usepackage{amsthm}

\newtheorem{proposition}{Proposition}

\newcommand*\dif{\mathop{}\!\mathrm{d}}

%% file: sec_intro.tex
\section{Introduction}

Policy gradient methods have been empirically very successful in many real-life reinforcement learning (RL) problems, including high-dimensional continuous control and video games \citep{schulman2017proximal,mnih2016asynchronous,sutton2000policy,williams1992simple}. Compared to alternative approaches to RL, policy gradient methods have several properties that are especially appealing to practitioners. For example, they directly optimize the performance criterion hence can continuously produce better policies. They do not require the knowledge of the environment dynamics (model-free) thus have broad applicability. They are typically on-policy, i.e., estimating the gradients with on-policy rollout trajectories, thus are easy to implement without the need of a replay buffer.

However, in real life, policy gradient methods have long suffered from the notorious issue of high gradient variance since Monte Carlo methods are used in policy gradient estimation \citep{sutton2000policy,sutton2018reinforcement}. Optimization theory dictates that the convergence speed of stochastic gradient descent crucially depends on the variance of the stochastic gradient \citep{ghadimi2016mini}. In principle, a lower variance policy gradient estimator should lead to higher sample efficiency.

Therefore, for a long time, people have investigated various ways to reduce the variance of policy gradient estimators. The control variates (CV) method, a well-established variance reduction technique in Statistics \citep{rubinstein1985efficiency,nelson1990control}, is a representative approach \citep{mnih2016asynchronous,schulman2015trust,schulman2017proximal,sutton2018reinforcement}. In order to reduce the variance of an unbiased Monte Carlo estimator, a control variate with known expectation under the randomness is subtracted from (and the expectation added back to) the original estimator. The resulting estimator would still be unbiased but have lower variance than the original estimator if the control variate correlates with the original estimator.

Researchers in reinforcement learning utilize a baseline function to construct the control variate. The term ``baseline function'' refers to the function that is subtracted from the return estimates when computing policy gradients. Consider a state dependent baseline function $b(\rvs)$, then the gradient estimator is $\nabla \log \pi(\rvs,\rva) \left( \hat{Q}(\rvs, \rva) - b(\rvs) \right)$. Here, the baseline $b$ corresponds to the control variate $ \nabla \log \pi(\rvs,\rva) b(\rvs) $. A frequently-used baseline is the approximate state value function $b(\rvs) = V(\rvs)$ due to its convenience and simplicity. The value function baseline is not optimal in terms of variance reduction though. The optimal state-dependent baseline is known as the squared-gradient-norm weighted average of the state-action values \citep{weaver2001optimal,greensmith2004variance,zhao2011analysis}.

Baseline functions and control variates have contributed to the success of policy gradients in modern deep RL \citep{schulman2017proximal,mnih2016asynchronous,chung2020beyond} and many works have followed up. For example, one line of follow-up work designs state-action dependent baselines rather than state-only dependent baselines to achieve further variance reduction \cite{liu2018action,grathwohl2018backpropagation,wu2018variance}. Nevertheless, prior works only consider the scalar-valued functions as baselines. Scalar-valued baselines ignore the differences between gradient coordinates. The fact that the policy gradient estimator is a random vector rather than a single random variable has been largely overlooked. Since the gradient is a vector, it is possible to employ a vector-valued function as the baseline to construct coordinate-wise control variates. In other words, it is feasible to apply a separate control variate for each coordinate of the gradient vector in order to achieve even further variance reduction.

In this paper, we investigate this new possibility of vector-valued baseline functions through extensive experiments on continuous control benchmarks. These new families of coordinate-wise and group-wise (layer-wise) control variates derived from the vector-valued baselines respect the differences between the parameter coordinates and sample points. We equip the popular Proximal policy optimization (PPO) algorithm \citep{schulman2017proximal} with the new control variates. The experiment results demonstrate that indeed further variance reduction and higher learning efficiency can be obtained with these more sophisticated control variates than the conventional scalar control variates.

%% file: sec_related.tex
\section{Related work}

\paragraph{Policy gradient (PG) methods.} There is a large body of literature on policy gradient methods  \citep{williams1992simple,sutton2000policy,konda2000actor,kakade2002natural,peters2006policy, mnih2016asynchronous,schulman2015trust,schulman2017proximal,gu2017interpolated,fakoor2020p3o}. They are a class of reinforcement learning algorithms that leverage gradients on a parameterized policy to do local policy search. In particular, the Policy Gradient Theorem proves the correctness of the general approach \citep{sutton2000policy}. The REINFORCE algorithm \citep{williams1992simple} can be thought of as PG with Monte Carlo estimated episodic return. The actor-critic variants (e.g., Advantage Actor-Critic) \citep{konda2000actor,mnih2016asynchronous} learn a value function approximation at the same time as policy learning, and adopt the bootstrapped value estimates rather than the (fully) Monte Carlo estimates to compute the gradients, which has led to better learning efficiency. Generalized Advantage Estimation (GAE) considers a geometric average of the bootstrapped values from different time steps as a low variance estimator for the advantage function \citep{schulman2015high} in a way similar to TD($\lambda$) \citep{sutton2018reinforcement}. Trust Region Policy Optimization \citep{schulman2015trust} and Proximal Policy Optimization \citep{schulman2017proximal} derive variants of PG by extending ideas from Conservative Policy Iteration \citep{kakade2002approximately} and Natural Policy Gradient \citep{kakade2002natural} into the deep learning regime.

\vspace{-5pt}
\paragraph{Variance reduction with control variates.}
The control variates method was seen in Monte Carlo simulation \citep{rubinstein1985efficiency,nelson1990control,glynn2002some} and finance \citep{glasserman2003monte}.
Its application in reinforcement learning or Markov decision processes dates back to at least \cite{weaver2001optimal,greensmith2004variance}. Although the most common choice of baseline in practice is the value function \cite{schulman2017proximal,mnih2016asynchronous}, it is known that the optimal scalar-valued state-dependent baseline is the Q values weighted by the squared gradient norms of the policy function \citep{weaver2001optimal,greensmith2004variance,zhao2011analysis,peters2006policy}.

Many have followed this line of work to construct better control variates. For example, \citet{gu2017q}, \citet{liu2018action}, \citet{grathwohl2018backpropagation} and \citet{wu2018variance} developed \emph{(state-)action-dependent} control variates to reduce the variance further, since the action-dependent CV can better correlate with original PG estimator. The action-dependent CVs are later more carefully analyzed by \citet{tucker2018mirage}. Recently, \citet{cheng2020trajectory} extended action-dependent CV to trajectory-wise CV by exploiting the temporal structure. An earlier report \citep{pankov2018reward} was motivated similarly and presented the similar main result to \citet{cheng2020trajectory}. Nevertheless, these work consider scalar-valued baseline functions, i.e., a single control variate for the whole PG. 
\citet{huang2020importance} showed many policy gradient estimators can be derived from finite difference of importance sampling evaluation estimators, and proposed a general form of PG that subsumes \citep{cheng2020trajectory} as a special case. Their additional term takes the form of gradient of state-action baselines which can be viewed as vector-valued CV.

An early work \citep{peters2006policy} described the possibility of closed-form coordinate-wise baselines in the context of robotics and \emph{shallow} policy parameterization (in fact, linear function as the mean of a Gaussian). Our paper differs from them in that (1) we demonstrate the effectiveness of coordinate CV in the context of deep neural nets and modern policy optimization algorithm (PPO) rather than vanilla policy gradients; (2) we study \emph{fitted} baselines instead of closed-form baselines; (3) we investigate the possibility of novel layer-wise CVs specifically for neural nets which is in-between coordinate-wise and scalar CVs.

\vspace{-5pt}
\paragraph{Other variance reduction techniques.} Besides the control variates method, there exist other tools to build low variance gradient estimators. Rao-Blackwellization is a classic information-theoretic variance reduction technique that has been applied in estimating gradients of discrete distribution \citep{liu2019rao}.
Others seek to apply well-established stochastic optimization techniques in convex optimization literature, such as the Stochastic Variance Reduced Gradient (SVRG) \citep{johnson2013accelerating,reddi2016stochastic} and the Stochastic Average Gradient (SAG) \citep{schmidt2017minimizing} to policy optimization \citep{papini2018stochastic} and deep Q-learning \citep{anschel2017averaged}. Variance reduction is also possible through action clipping if we allow the final estimator to be biased \citep{cheng2019control}.
AVEC \citep{flet2021learning} proposed to use the residual variance as an objective function for the critic, which leads to better value function estimates and better low variance gradient estimator.

%% file: sec_method.tex
\setlength{\textfloatsep}{1\baselineskip plus 0.2\baselineskip minus 0.7\baselineskip}

\section{Method}
\label{sec:method}

We review the variance reduction with control variates technique in policy gradient methods and introduce the coordinate-wise control variates and demonstrate how to integrate them with PPO.

\subsection{Background}

Reinforcement learning (RL) is often formalized as maximizing the expected cumulative return obtained by an agent in a Markov decision process $(\gS, \gA, \gP, r)$. Here the set $\gS$ is the state space, the set $\gA$ is the action space, $\gP(\rvs_{t+1} | \rvs_t, \rva_t)$ is the (probabilistic) transition function (dynamics), and $r(\rvs, \rva)$ is the reward function. The infinite-horizon $\gamma$-discounted return setting defines the expected cumulative return as
\begin{equation}
    J = \E \left[ \sum_{t=0}^{\infty} \gamma^t r(\rvs_t, \rva_t) \right] .
\end{equation}

The policy gradient method is a popular class of model-free RL methods that directly optimizes the return of a parameterized policy $\pi_\theta$, by treating RL as a stochastic optimization problem and applying gradient ascent. The gradient is obtained by the Policy Gradient Theorem \citep{sutton2000policy} as in Eq.~\ref{eq:pg}.
\begin{align}
    \label{eq:pg}
    \nabla_\theta J(\theta) = \E_{\rvs,\rva \sim \pi_\theta} \nabla_\theta \log \pi_\theta(\rva|\rvs) Q_{\pi_\theta}(\rvs, \rva)
    .
\end{align}
The shorthand notation $\rvs,\rva \sim \pi_\theta$ denotes that the expectation is taken over the state-action distribution induced by $\pi_\theta$. Specifically, $\rvs$ is drawn from the (un-normalized) discounted state distribution $\sum_{t=0}^\infty \gamma^t \Pr[\rvs_t = s]$\footnote{If we instead define the normalized state distribution by multiplying $1-\gamma$, there needs to be a $\frac{1}{1-\gamma}$ coefficient in front of the policy gradient in Eq.~\ref{eq:pg}.} and $\rva|\rvs$ is drawn from the conditional action distribution $\pi_\theta(a|\rvs)$. $Q_{\pi_\theta}$ refers to the state-action value function associated with the policy $\pi_\theta$. Given the Q value estimates $\hat{Q}_{\pi_\theta}$, one can easily construct the following Monte Carlo estimator of the policy gradient with a state-action pair sample (or with $n$ i.i.d. state-action pairs and the average of the same constituting estimator).
\begin{align}
    \label{eq:g_0}
    \rvg^0 = \nabla_{\theta} \log \pi_\theta(\rva|\rvs) \hat{Q}_{\pi_\theta}(\rvs, \rva)
    \\
    \text{ where }
    \rvs,\rva \sim \pi_\theta
    \nonumber
\end{align}

Optimization theory predicts that the convergence rate of stochastic gradient descent depends crucially on the variance of the gradient estimator. For example, \citet{ghadimi2016mini} states that the iteration complexity of mirror descent methods to an $\eps$-approximate stationary point in non-convex problems follows
$
    O\left( \frac{ \sigma^2 }{\eps^2} \right)
$
where $\sigma^2 = \Tr \Var[\text{gradient}]$.

Therefore, for decades, researchers have studied various ways to reduce the variance of the policy gradient estimators. The control variates (CV) method is a prominent one. In its generic form, suppose $g$ is the original unbiased Monte Carlo estimator, $c$ is a control variate which correlates with $g$ and has known expectation, then $g' = g - c + \E[c]$ is an unbiased estimator to the same quantity as $g$ but with lower variance. In the policy optimization context, a new estimator is constructed in Eq.~\ref{eq:g_b} by subtracting, for instance, a state-dependent baseline function $b(\rvs)$ from the Q value estimation term. This effectively constructs the control variate $\nabla_{\theta} \log \pi_\theta(\rva|\rvs) b(\rvs) $.
\begin{align}
    \label{eq:g_b}
    \rvg^b = \nabla_{\theta} \log \pi_\theta(\rva|\rvs) (\hat{Q}_{\pi_\theta}(\rvs, \rva) - b(\rvs))
    \\
    \text{ where }
    \rvs,\rva \sim \pi_\theta
    \nonumber
\end{align}

This is valid because, for any real-valued function $b(\rvs)$, the expectation of the control variate is
\begin{align*}
    &\; \E_{\rvs} \left[ \E_{\rva | \rvs \sim \pi_\theta} \left[ \nabla_{\theta} \log \pi_\theta(a | \rvs) b(\rvs) \right] \right]  \\
    =&\; \E_{\rvs} \left[ b(\rvs) \int_{a} \pi_\theta(a | \rvs) \nabla_{\theta} \log \pi_\theta(a | \rvs) \dif a \right]  \\
    =&\; \E_{\rvs} \left[ b(\rvs) \int_{a} \nabla_{\theta} \pi_\theta(a | \rvs) \dif a \right]  \\
    =&\; \E_{\rvs} \left[ b(\rvs) \nabla_{\theta} \left( \int_{a}  \pi_\theta(a | \rvs) \dif a \right) \right]  \\
    =&\; \E_{\rvs} \left[ b(\rvs) \nabla_{\theta} 1 \right] = 0
    .
\end{align*}

A common choice of baseline function in practice \citep{schulman2015trust,mnih2016asynchronous,schulman2017proximal} is the estimated version $\hat{V}_{\pi_\theta}$ of the expected value
$ V_{\pi_\theta}(\rvs) = \E_{\rva|\rvs \sim \pi_\theta} Q_{\pi_\theta}(\rvs, \rva) $ through function approximation due to simplicity. We call this estimator with value function baseline as $\rvg^v$ in Eq.~\ref{eq:g_v}.
\begin{align}
    \label{eq:g_v}
    \rvg^v = \nabla_{\theta} \log \pi_\theta(\rva|\rvs) (\hat{Q}_{\pi_\theta}(\rvs, \rva) - \hat{V}_{\pi_\theta}(\rvs))
    \\
    \text{ where }
    \rvs,\rva \sim \pi_\theta
    \nonumber
\end{align}

\subsection{Minimum variance control variates}

The variance of a $d$-dimensional policy gradient estimator $\rvg$ can be defined in one way as the trace of the (co)variance matrix, i.e.
\begin{equation}
    \label{eq:var_g}
    \sV[\rvg] = \Tr \Var[\rvg] = \sum_{j=1}^d \Var[\rvg_j]
    = \sum_{j=1}^d \E[\rvg_j^2] - \E[\rvg_j]^2
\end{equation}

A minimum variance control variate can be found by explicitly minimizing $\sV[\rvg^b]$ with respect to $b$. After inspection, one can notice that only the first term of Eq.~\ref{eq:var_g} depends on $b$ while the second term does not. By minimizing $\sum_{j=1}^d \E[{\rvg^b_j}^2]$ w.r.t. $b$, we can easily arrive at the following optimal scalar-valued state-dependent baseline \citep{weaver2001optimal,greensmith2004variance,zhao2011analysis}.
\begin{equation}
\label{eq:g_squared}
\begin{aligned}
    \sum_{j=1}^d \E[{\rvg^b_j}^2] 
    & = \E \sum_{j=1}^d \big( \frac{\partial}{\partial \theta_j} \log \pi_\theta(\rva|\rvs) \big)^2
    \big( \hat{Q}_{\pi_\theta}(\rvs, \rva) - b(\rvs) \big)^2
    \\
    & = \E\left[ \| \nabla_{\theta} \log \pi_\theta(\rva | \rvs) \|^2
    \big( \hat{Q}_{\pi_\theta}(\rvs, \rva) - b(\rvs) \big)^2 \right]
    .
\end{aligned}
\end{equation}
\begin{equation}
\label{eq:best_scalar_b}
\begin{aligned}
    &\text{For any state $\rvs$,  } \frac{\partial}{\partial b(\rvs)} \sum_{j=1}^d \E[{\rvg^b_j}^2] = 0  \implies\\
    &b^*(\rvs) = \frac{\E_{\rva|\rvs \sim \pi_\theta} \left[
    \| \nabla_{\theta} \log \pi_\theta(\rva | \rvs) \|^2
    \hat{Q}_{\pi_\theta}(\rvs, \rva)
    \right] }
    {\E_{\rva|\rvs \sim \pi_\theta} \| \nabla_{\theta} \log \pi_\theta(\rva | \rvs) \|^2 }
    .
\end{aligned}
\end{equation}

\subsection{Coordinate-wise and layer-wise control variates}

Now, is Eq.~\ref{eq:best_scalar_b} truly optimal? The answer is yes, in the sense if we restrict ourselves to scalar-valued baseline functions $b(\rvs): \gS \mapsto \sR$, which is essentially assigning the same control variate to all coordinates of the gradient vector. However, scalar-valued baselines disrespect the difference between coordinates. If we relax the baseline function space to vector-valued functions, namely $c(\rvs) = \left( c_1(\rvs), c_2(\rvs), \ldots, c_d(\rvs) \right): \gS \mapsto \sR^d$ where $d$ is the dimension of the gradient, the gradient variance could potentially be further reduced.

Consider the following family of policy gradient estimators with coordinate-wise control variates,
\begin{equation}
\begin{aligned}
    \label{eq:g_coord}
    \rvg^c = (\rvg^c_1, \rvg^c_2, \ldots, \rvg^c_d) ,
    \\
    \rvg^c_j = \frac{\partial}{\partial \theta_j} \log \pi_\theta(\rva|\rvs) \left( \hat{Q}_{\pi_\theta}(\rvs, \rva) - c_j(\rvs) \right)
    \\
    \text{ where }
    \rvs,\rva \sim \pi_\theta
    .
\end{aligned}
\end{equation}

We assign a separate baseline function $c_j(\rvs), j=1\ldots d$ for each coordinate of the gradient vector, or equivalently, a vector-valued baseline $c(\rvs)$ for the whole gradient. Since the space of feasible solutions is enlarged, the minimum variance estimator in this family would be able to achieve a lower variance than the scalar-value counterpart. 
Without any restrictions on function space containing $c_j$, the optimal vector-valued baseline admits a form similar to Eq.~\ref{eq:best_scalar_b}, except that now the weights are the squared partial derivatives instead of the gradient norm:
$$
c_j^*(\rvs) = \frac{\E_{\rva|\rvs \sim \pi_\theta} \left[
    \big( \frac{\partial \log \pi_\theta(\rva | \rvs)}{\partial\theta_j} \big)^2
    \hat{Q}_{\pi_\theta}(\rvs, \rva)
    \right] }
    {\E_{\rva|\rvs \sim \pi_\theta} \big( \frac{\partial \log \pi_\theta(\rva | \rvs)}{\partial\theta_j} \big)^2 } .
$$

When function approximation is used, we have the following proposition that formalizes the intuition and suggests that further variance reduction could be achieved with more sophisticated control variates.

\begin{proposition}[Variance comparison]
    \label{prop:var}
    Given a real-valued function class $\gF: \gS \mapsto \sR $ for $b, c_j$ and $\hat{V}_{\pi_\theta}$. Let the optimal scalar-valued baseline be $$b^* = \argmin_{b\in \gF} \sV[\rvg^b]$$ and the optimal vector-valued baseline be $$c^* = \argmin_{c \in \gF^d} \sV[\rvg^{c}] . $$
    The variances of the policy gradient estimators with the optimal coordinate-wise CV, the optimal scalar CV, and the value function CV satisfy the following relationship:
    \begin{equation*}
        \sV[\rvg^{c^*}] \le \sV[\rvg^{b^*}] \le \sV[\rvg^v] 
        .
    \end{equation*}
\end{proposition}
\begin{proof}
    The proposition is a straightforward consequence of the definition of minimizers.
    
    Comparing $\rvg^c$ and $\rvg^b$, since $c_j$ and $b$ belong to the same basic function class, $c' = (b^*,b^*, \ldots, b^*)$ is a special case such that $c' \in \gF^d$.
    This means the coordinate-wise CV is more expressive than the scalar-valued CV.
    Furthermore, the minimizer $c^*$ in this more expressive function class will be able to achieve a smaller objective value than $c'$. Hence, $\sV[\rvg^{c^*}] \le \sV[\rvg^{c'}] = \sV[\rvg^{b^*}]$.
    
    Comparing $\rvg^{b*}$ and $\rvg^v$, this is trivially true because $ \sV[\rvg^{b^*}] = \min_{b\in \gF} \sV[\rvg^b] \le \sV[\rvg^b], \forall b \in \gF $ and $\hat{V}_{\pi_\theta} \in \gF$ by assumption.
\end{proof}

A middle ground between the scalar-valued and the fully coordinate-wise control variates is the group-wise version. The coordinates of the gradient vector can be partitioned into disjoint groups. Within each group, the same control variate is shared. A natural choice of grouping strategy is to partition by the layers in a neural network (i.e., parameter tensors). We refer to this variant as the \textbf{Layer-wise CV}.

In practice, we may choose neural nets as the baseline function class in our experiments. We build a fully-connected neural net with the same architecture as the value function except for the last layer. For the coordinate-wise baseline function, the last layer outputs $d$ predictions where $d$ is the number of policy parameter coordinates. For the layer-wise baseline function, the number of outputs equals the number of layer weight tensors.

\subsection{Fitting a neural vector-valued baseline}

We use the following loss function to train the vector-valued baseline function $c_{\phi} = (\ldots, c_{j,\phi}, \ldots)$ parameterized by $\phi$. The general form of this loss takes two flexible hyper-parameters $\lambda$ and $\rho$, where $\lambda$ plays the role of interpolating between fitting a value function and minimizing variance, and $\rho$ imposes a proximal regularization term inspired by the original PPO. Intuitively, this loss is still fitting $c$ to the Q values but weighted by the squared derivatives. A coordinate $j \in \{1,2,\ldots,d\}$ or data point $i \in \{1,2,\ldots,n\}$ ($n$ is the number of state-action pairs) with large derivative values may get more weights to reduce the overall gradient variance.
\begin{equation}
\label{eq:l_baseline}
\begin{aligned}
    L^{\text{baseline}}
    = 
    & \frac1{nd} \sum_{i=1}^n \sum_{j=1}^d 
    \bigg\{ 
    \big( \hat{Q}_{\pi_\theta}(\rvs_i, \rva_i) - c_{j,\phi}(\rvs_i) \big)^2
    \\
    & \quad
    \underbrace{
      \cdot 
      \Big( (1-\lambda)
        \overline{ \big( \frac{\partial}{\partial \theta_j} \log \pi_\theta(\rva_i|\rvs_i) \big)^2 }
        + \lambda \Big)
    }_{\text{$\lambda$-interpolated Empirical Variance}}
    \\
    & +
    \underbrace{
      \rho \cdot \big( c_{j,\phi}(\rvs_i) - c_{j,\phi_{\text{old}}}(\rvs_i) \big)^2
    }_{\text{Proximal Regularization}}
    \bigg\}
\end{aligned}
\end{equation}

The over-line in $\overline{ \big( \frac{\partial}{\partial \theta_j} \log \pi_\theta(\rva_i|\rvs_i) \big)^2 }$ indicates that the derivative square is normalized to have 1 as its mean value in a mini-batch, namely $\overline{ \big( \frac{\partial}{\partial \theta_j} \log \pi_\theta(\rva_i|\rvs_i) \big)^2 } = \frac{\big( \frac{\partial}{\partial \theta_j} \log \pi_\theta(\rva_i|\rvs_i) \big)^2}{\frac1n \sum_i \big( \frac{\partial}{\partial \theta_j} \log \pi_\theta(\rva_i|\rvs_i) \big)^2}$.
In this way, the loss function is properly normalized: It is only affected by the relative difference between different coordinates but not the absolute scale of the gradient.
As a special case of Eq.~\ref{eq:l_baseline} when $\lambda=1$, there is no difference between the losses for different $j$, the objective degenerates into fitting a conventional value function approximation. When $\lambda \to 0$, the loss stays more faithfully to the definition of gradient variance.

Since we do mini-batch gradient descent to optimize Eq.~\ref{eq:l_baseline} in practice, using the original definition of variance as objective might not be desirable when the mini-batch sample is small. The new $c_\phi$ might easily over-fit to the current sample due to the large number of free parameters. The proximal term is out of this consideration. This is inpired by PPO, as the original PPO employs a similar kind of proximal methods to update the value function approximator \citep{schulman2017proximal}. When $\rho$ is large, the regularization term encourages the new baseline function to stay close to the old one and induces conservative updates to the baseline parameters.

\begin{algorithm}[t]
\small
Initialize policy $\pi_\theta$, baseline $c_\phi$, and value function\;
\For{$i = 0,1,2,\ldots$} {
    Collect rollout trajectory sample $D = \left\{s_i, a_i, r_i\right\}_{i=1}^n$ with policy $\pi_\theta$\;
    Compute TD($\lambda$) return estimates $\hat{Q}_{\pi_\theta}$ with $D$ and the fitted value function\;
    $\theta_\text{old} \leftarrow \theta, \phi_\text{old} \leftarrow \phi$\;
    Update the vector-valued baseline $c_\phi$ by (Eq.~\ref{eq:l_baseline})
    $\phi \leftarrow \argmin_{\phi} L^{\text{baseline}}(\phi, \phi_{\text{old}}, D, \hat{Q}_{\pi_\theta}) \text{;}$
    \par
    Compute per-coordinate advantages $\hat{A}_j$\ by Eq.~\ref{eq:a_j}\;
    Update the policy with surrogate loss Eq.~\ref{eq:l_ppo_j} by
    $\theta \leftarrow \argmin_{\theta} L^{\text{ppo}}(\theta, \theta_\text{old}, D, \hat{Q}_{\pi_\theta}, \hat{A}_j)  \text{;}$
    \par
    Update the value function as in the original PPO\;
}
\Return $\pi_\theta$\;
\caption{PPO with coordinate-wise CV.}
\label{alg:ppo}
\end{algorithm}

\subsection{Integration with Proximal Policy Optimization}
\label{sec:method_ppo}

We show how to integrate the coordinate-wise control variates with the popular Proximal Policy Optimization (PPO) algorithm \citep{schulman2017proximal}. Since PPO introduces several modifications over the vanilla policy gradient, applying the coordinate-wise control variates requires some non-trivial redesign of the PPO objective functions, as we shall see in Sec.~\ref{sec:method_ppo_sub2}.

\subsubsection{The original PPO objective}

Given a trajectory sample of $n$ state-action-reward tuple $D = \{(s_i,a_i,r_i)\}_{i=1}^n$, PPO defines a surrogate objective motivated by proximal/trust-region methods for policy improvement:
\begin{equation}
\label{eq:ppo}
\begin{aligned}
    L^{\text{ppo}}
    &= \frac1n \sum_{i=1}^n \bigg[
        \min \Big\{
             \frac{\pi_\theta(a_i|s_i)}{\pi^{\text{old}}(a_i|s_i)} \hat{A}(s_i,a_i), \\
          &  \hspace{3em} \Big[\frac{\pi_\theta(a_i|s_i)}{\pi^{\text{old}}(a_i|s_i)}\Big]_{1-\eps}^{1+\eps} \hat{A}(s_i,a_i)
        \Big\}
    \bigg]
    .
\end{aligned}
\end{equation}

In the objective, $\pi^\text{old}$ is the old policy (fixed), $\pi_\theta$ is the policy to be optimized, and $\hat{A}$ is the estimated advantage function usually computed with the Generalized Advantage Estimation (GAE) \citep{schulman2015high}. PPO further defines the following clipped importance sampling ratio:
\begin{align*}
    \omega_i = \omega(a_i|s_i) = \begin{cases}
        \left[ \frac{\pi_\theta(a_i|s_i)}{\pi^{\text{old}}(a_i|s_i)} \right]^{1+\eps} & \text{if $\hat{A}(s_i,a_i) \ge 0$;}
        \\
        \vspace{-6pt}
        \\
        \left[ \frac{\pi_\theta(a_i|s_i)}{\pi^{\text{old}}(a_i|s_i)} \right]_{1-\eps} & \text{if $\hat{A}(s_i,a_i) < 0$.}
    \end{cases}
\end{align*}

Here the notation $[\;\cdot\;]^{1+\eps} = \min\{\;\cdot\;, 1+\eps\}$, $[\;\cdot\;]_{1-\eps} = \max\{\;\cdot\;, 1-\eps\}$, and $[\;\cdot\;]^{1+\eps}_{1-\eps} = \big[[\;\cdot\;]^{1+\eps}\big]_{1-\eps}$.

Moreover, note that the advantage function is implicitly using the value function as the baseline function, i.e., $\hat{A}(s_i,a_i) = \hat{Q}_{\pi_\theta}(s_i,a_i) - \hat{V}_{\pi_\theta}(s_i)$. Computing the GAE $\hat{A}$ is equivalent to computing the TD($\lambda$) for $\hat{Q}$ \citep{schulman2015high,sutton2018reinforcement}.
Then the PPO objective can be rewritten as the weighted sum 
\begin{align*}
    L^{\text{ppo}}
    &= \frac1n \sum_{i=1}^n \left[\omega(a_i|s_i) \hat{A}(s_i,a_i) \right] \\
    &= \frac1n \sum_{i=1}^n \left[\omega(a_i|s_i) \left( \hat{Q}_{\pi_\theta}(s_i,a_i) - \hat{V}_{\pi_\theta}(s_i) \right) \right] .
\end{align*}

Notice that the gradient of the PPO objective for the policy parameter theta can also be expressed in the following compact form using the clipped ratio
\begin{align*}
    \nabla_\theta L^{\text{ppo}}
    &= \frac1n \sum_{i=1}^n \left[ \omega_i \1_{\omega_i \ne 1\pm \eps} \nabla_\theta \log \pi_\theta(a_i|s_i) \hat{A}(s_i,a_i) \right] \\
    &= \frac1n \sum_{i=1}^n \bigg[ \omega_i \1_{\omega_i \ne 1\pm \eps} \cdot \nabla_\theta \log \pi_\theta(a_i|s_i) \\
    & \hspace{5em} \cdot \left( \hat{Q}_{\pi_\theta}(s_i,a_i) - \hat{V}_{\pi_\theta}(s_i) \right) \bigg]
    .
\end{align*}

This now resembles very much to the policy gradient estimators Eq.~\ref{eq:g_v} and Eq.~\ref{eq:g_b} except for the importance weights $\omega_i \1_{\omega_i \ne 1\pm \eps}$. The indicator function $\1_{\omega_i \ne 1\pm \eps}$ occurs due to the way how AutoGrad works. When a sample point is clipped, it does not contribute to the final gradient calculation, as the AutoGrad engine regards $1 \pm \eps$ as constants without gradients. Hence, PPO can be thought of as adaptively dropping sample points to avoid extreme gradient values. In practice, only a small fraction of sample points are clipped and the remaining $\omega_i$ are close to 1. 
Therefore, we can still leverage an empirical variant to Eq.~\ref{eq:g_squared} to fit a baseline function other than $\hat{V}_{\pi_\theta}$ in the hope to reduce the final gradient variance.

\subsubsection{PPO + Coordinate-wise CV}
\label{sec:method_ppo_sub2}

When using the coordinate-wise control variates, i.e., using the corresponding vector-valued baseline function), conceptually, a separate clipped objective is defined for each parameter coordinate. Note the difference from the previous equations in subscripts $j$.

For any $i \in \{1,2,\ldots,n\}$ and $j \in \{1,2,\ldots,d\}$, we analogously have the following set of equations.
\begin{align*}
    \omega_{ij} = \omega_j(a_i|s_i) = \begin{cases}
        \left[ \frac{\pi_\theta(a_i|s_i)}{\pi^{\text{old}}(a_i|s_i)} \right]^{1+\eps} & \text{if $\hat{A}_j(s_i,a_i) \ge 0$;}
        \\
        \vspace{-6pt}
        \\
        \left[ \frac{\pi_\theta(a_i|s_i)}{\pi^{\text{old}}(a_i|s_i)} \right]_{1-\eps} & \text{if $\hat{A}_j(s_i,a_i) < 0$.}
    \end{cases}
\end{align*}
\vspace{-5pt}
\begin{align}
\label{eq:a_j}
    \hat{A}_j(s_i,a_i) = \hat{Q}(s_i,a_i) - c_{j,\phi}(s_i). 
\end{align}
\vspace{-5pt}
\begin{align}
\label{eq:l_ppo_j}
    L^{\text{ppo}}_j
    = \frac1n \sum_{i=1}^n \left[\omega_j(a_i|s_i) \hat{A}_j(s_i,a_i) \right]
    .
\end{align}

This looks cumbersome to compute, as constructing $d$ separate loss functions and doing back-propagation is computationally too expensive. Thankfully, if we inspect the formula for the new objectives' gradients, we notice that the $\frac{\partial}{\partial\theta_j} \log \pi_\theta $ part only needs to be computed once.
\begin{align*}
    \frac{\partial L^{\text{ppo}}}{\partial \theta_j}
    &= \frac1n \sum_{i=1}^n \left[ \omega_{ij} \1_{\omega_{ij} \ne 1\pm \eps} \frac{\partial}{\partial\theta_j} \log \pi_\theta(a_i|s_i) \hat{A}_j(s_i,a_i) \right] \\
    &= \frac1n \sum_{i=1}^n \bigg[ \omega_{ij} \1_{\omega_{ij} \ne 1\pm \eps} \cdot \frac{\partial}{\partial\theta_j} \log \pi_\theta(a_i|s_i) \\
    & \hspace{5em} \cdot \left( \hat{Q}_{\pi_\theta}(s_i,a_i) - c_{j,\phi}(s_i) \right) \bigg]
    .
\end{align*}

Leveraging efficient per-example AutoGrad implementations in deep learning libraries such as PyTorch \citep{paszke2019pytorch}, we can compute the $\nabla \log \pi_\theta(a_i|s_i)$ for all $(s_i,a_i)$ examples in one backward pass. Then the policy gradient is simply the matrix multiplication between $\nabla \log \pi_\theta$ and the per-coordinate advantages. Putting everything together, the modified PPO algorithm with coordinate-wise CV is outlined in Alg.~\ref{alg:ppo}.

There is one caveat here: In theory, the samples used to fit the baseline function and to optimize the policy need to be independent to make policy gradient unbiased. However, in practical implementations of PPO \cite{schulman2017proximal,liu2018action}, this requirement is often loosened since (1) The state-action values are computed \emph{before} the policy optimization; (2) Mini-batching is employed to train the baselines, the value function and the policy function jointly or separately; The stochasticity in the process might decorrelate the data points; (3) PPO is often interpreted as a policy optimization algorithm rather than a vannila policy gradient method; It works as long as local policy improvement can be achieved.
We indeed find during our experimentation that the choice of dependent or independent samples for baseline and policy training do not significantly affect performance.

%% file: sec_exp.tex
\section{Experiment}

We first conduct a case study to demonstrate the benefit of treating the coordinates differently and verify the variance reduction with more sophisticated CVs. We then demonstrate that PPO equipped with the proposed CVs has better sample efficiency in simulated control tasks.

\subsection{Case study in Walker2d}

\begin{figure}[tb]
    \centering
    \includegraphics[width=\linewidth]{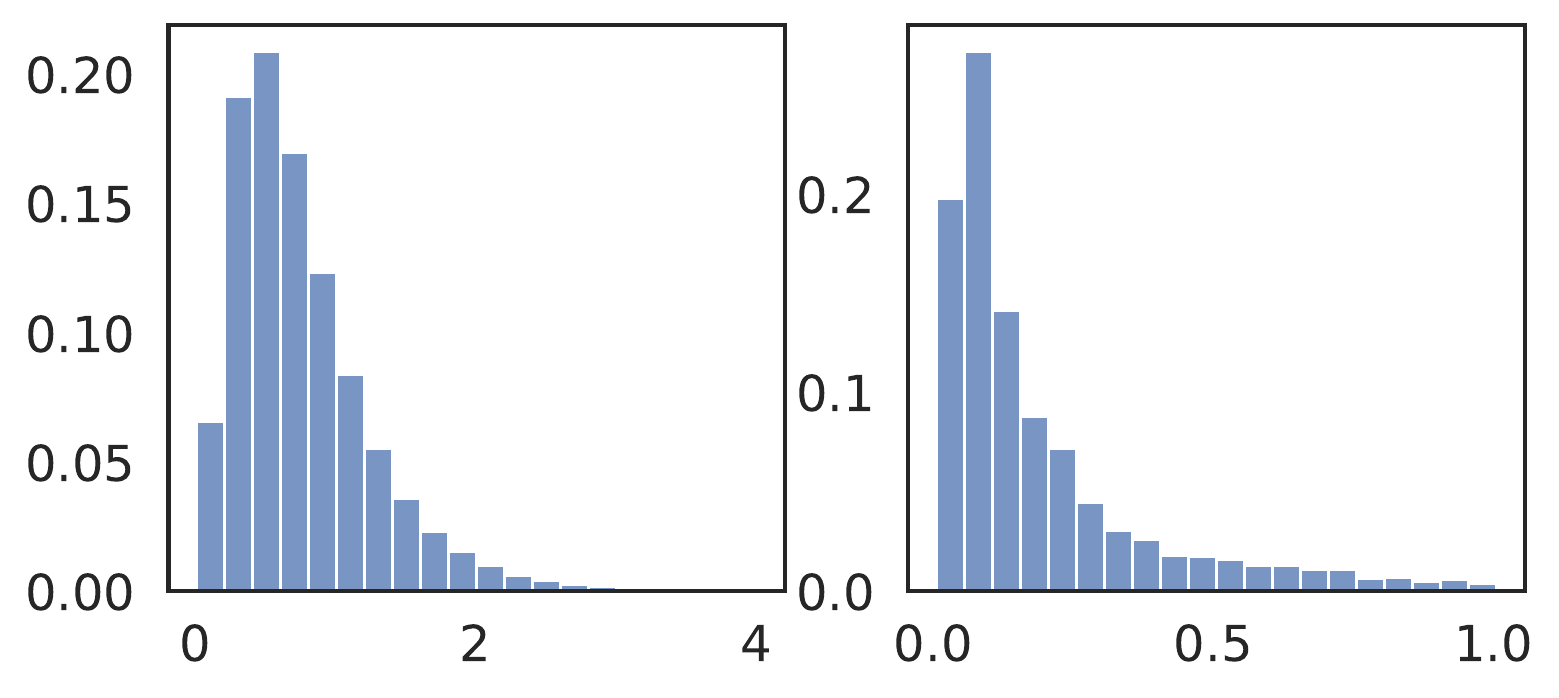}
    \caption{The histograms of $\frac1d \sum_{j=1}^d |\frac{\partial \log \pi_\theta(a_i|s_i)}{\partial \theta_j}|^2$ (left) and $\frac1n \sum_{i=1}^n |\frac{\partial \log \pi_\theta(a_i|s_i)}{\partial \theta_j}|^2$ (right). The squared gradient norms exhibit a tailed distribution. The coordinate-wise CV may benefit from using these norms as weights in a weighted regression.}
    \label{fig:grad_norm}
    \vspace{-10pt}
\end{figure}

\begin{table}[tb!]
    \centering
    \small
    \caption{The empirical variance of different policy gradient estimators derived from different control variates. The smaller the better. The gradients are computed with a fixed policy trained for 300 steps in the Walker2d environment. Value CV refers to using the value function learned previously during RL training as baseline. The other CVs (Value CV (refit), Scalar, Layer-wise and Coord.-wise) are \emph{refitted} with a \emph{new sample} of $10^5$ timesteps with the policy function fixed. Value CV (refit) refers to refitting a new value function as the baseline. Scalar, Layer-wise and Coord.-wise refer to using single baseline function for all coordinates, one baseline per neural net layer, and one baseline per coordinate, respectively, as described in Sec.~\ref{sec:method}.}
    \label{tab:var}
    \begin{tabular}{l|l}
        \toprule
        \bfseries{CV Method}    &  \bfseries{Variance estimates}  \\
        \midrule
        Without CV     &  $ 33355   \quad\; [32844, 33878] $      \\
        Value CV       &  $ 34.759  \quad [34.226, 35.304] $  \\
        Value CV (refit)& $ 22.333  \quad [21.990, 22.683] $    \\
        Scalar CV      &  $ 21.856  \quad [21.521, 22.199] $  \\
        Layer-wise CV  &  $ 21.405  \quad [21.077, 21.741] $  \\ 
        Coord.-wise CV &  $ \textbf{21.259}  \quad [20.933, 21.593] $ \\
        \bottomrule
    \end{tabular}
\end{table}

We take the Walker2d-v2 environment in the OpenAI Gym \citep{brockman2016openai} benchmark suite as the testbed to study the variances of policy gradient estimators with different control variates. Specifically, we pull out a policy checkpoint saved during the training of the regular PPO algorithms, then fit different baseline functions according to Eq.~\ref{eq:l_baseline} (with $\lambda=0.1, \rho=0$) for the fixed policy with a large sample ($10^5$ timesteps). After that, we report the empirical variance measured on another large roll-out sample ($10^5$ timesteps).

We follow the standard approach to estimate the variance by $\widehat{\sigma^2} = \frac1{n-1} \sum_{i=1}^n (g_i - \bar{g})^2 $ and the 95\% confidence intervals by the Chi-Square distribution.

Tab.~\ref{tab:var} shows that more variance reduction can be achieved with more sophisticated control variates for a fixed policy. In particular, the value function baseline is already able to offer a significant amount of variance reduction compared to the policy gradient without CV (Eq.~\ref{eq:g_0}). The layer-wise and coordinate-wise control variate can further reduce the variance by a decent amount. In the supplementary material, we also compare the Mean Squared Error (MSE) of different gradient estimators as they could be biased in practice. We find that more advanced CVs can also reduce the MSE of the gradients. These observations give us hope that the proposed control variates are \emph{capable} of reducing variances and might lead to improved sample efficiency, which we study in the next subsection.

One main difference between coordinate-wise CV and value CV is that the former leverages a weighted regression objective. 
Therefore, we plot the histograms of the weights, i.e., the squared derivative norms of $\log \pi_\theta$ in Fig.~\ref{fig:grad_norm}. Intuitively, weighing the baseline fitting errors by these gradient/derivative norms focuses the training effort onto data points and coordinates that have the highest impact on the policy gradient variance.

\begin{figure*}[tb]
    \centering
    \includegraphics[width=1.0\linewidth]{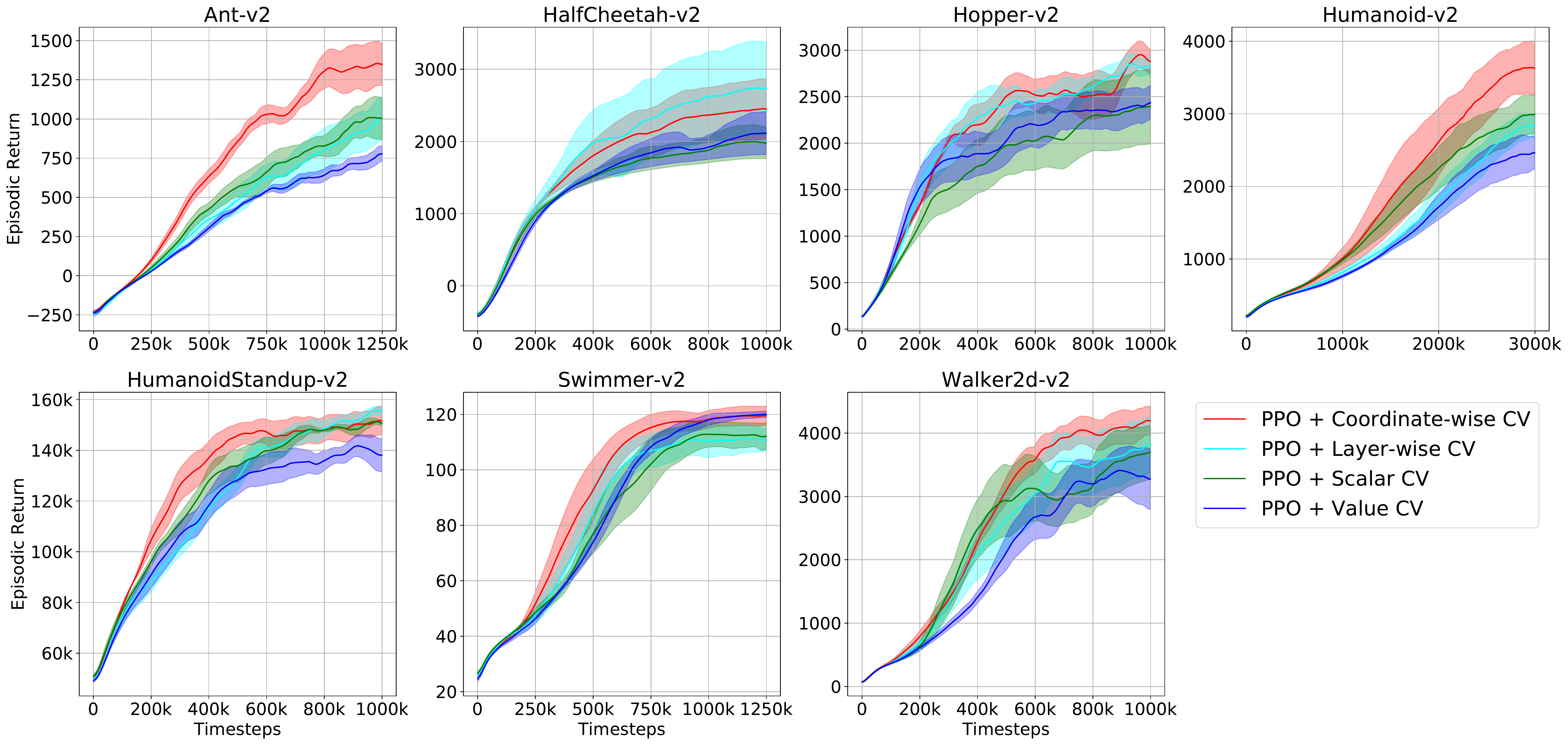}
    \caption{The episodic returns during training, averaged over 4 random seeds. The shades show $\pm 1$ standard error. The horizontal axis represents the number of environment timesteps. We ran 1M timesteps in total except for Humanoid, which requires 3M timesteps. In each subplot, other than the control variates, the rest of the PPO algorithm was the same. The curves are resampled to 512 points and smoothed with a window size 15.}
    \label{fig:mujoco}
    \vspace{-10pt}
\end{figure*}

\begin{table*}[tb!]
    \centering
    \caption{The mean episodic returns of \emph{all episodes} during training by different methods on the MuJoCo continuous control tasks. We compare (PPO plus) the value function CV (the original PPO), the fitted optimal scalar-valued CV, the proposed fitted layer-wise and coordinate-wise CVs. The results are averaged over 4 random seeds, and $\pm$ denotes the 1 standard error. The last column `Improve.' lists the normalized relative improvement over the first row (PPO+Value CV) averaged across the environments. }
    \label{tab:mujoco}
    \small
    \setlength{\tabcolsep}{2.4pt}
    \begin{tabular}{l|lllllll|r}
        \toprule
        {CV} & {Ant} & {HalfCheetah} & {Hopper} & {Humanoid} & {HumanoidStand} & {Swimmer} & {Walker2d} & {Improve.}  \\
        \midrule
        Value  & $156.2 \pm 5.0$   & $1418.3 \pm 128.3$ & $1297.9 \pm 115.9$ & $860.4 \pm 22.1$  & $115451.0 \pm 4934.5$ & $84.4 \pm 1.7$ & $1159.1 \pm 80.4$  & $+0\%$ \\
        Scalar & $210.3 \pm 38.5$  & $1399.9 \pm 115.5$ & $1160.3 \pm 126.2$ & $1034.9 \pm 48.0$ & $122945.1 \pm 2464.2$ & $82.2 \pm 4.8$ & $1337.6 \pm 165.2$ & $+8.9\%$ \\
        \midrule
        Layer  & $189.5 \pm 50.7$  & $1833.4 \pm 473.6$ & $1395.8 \pm 97.3$  & $902.8 \pm 39.9$  & $120230.0 \pm 3109.1$ & $84.3 \pm 5.5$ & $1325.7 \pm 154.3$ & $+11.6\%$ \\
        Coord  & $323.7 \pm 13.8$  & $1694.5 \pm 250.3$ & $1408.6 \pm 28.1$  & $1056.1 \pm 97.4$ & $128267.3 \pm 4620.5$ & $91.0 \pm 4.5$ & $1455.9 \pm 73.7$  & $+28.9\%$ \\
        \bottomrule
    \end{tabular}
    \caption{The mean episodic returns of the \emph{last 100 episodes} rather than all episodes.}
    \label{tab:mujoco100}
    \setlength{\tabcolsep}{1.7pt}
    \begin{tabular}{l|lllllll|r}
        \toprule
        {CV} & {Ant} & {HalfCheetah} & {Hopper} & {Humanoid} & {HumanoidStand} & {Swimmer} & {Walker2d} & {Improve.}  \\
        \midrule
        Value  & $755.3 \pm 58.9$   & $2110.1 \pm 341.2$ & $2420.0 \pm 216.0$ & $2569.2 \pm 184.8$ & $140955.1 \pm 6520$ & $119.9 \pm 1.3$ & $3286.5 \pm 532.7$ & $+0\%$ \\
        Scalar & $998.7 \pm 144.5$  & $1991.1 \pm 263.0$ & $2378.4 \pm 460.2$ & $3038.0 \pm 306.1$ & $150306.6 \pm 2425$ & $112.6 \pm 5.0$ & $3688.2 \pm 487.6$ & $+8.0\%$ \\
        \midrule
        Layer  & $966.8 \pm 152.2$  & $2735.7 \pm 757.7$ & $2849.2 \pm 113.7$ & $2911.0 \pm 172.1$ & $153919.0 \pm 2329$ & $111.1 \pm 5.3$ & $3789.4 \pm 498.0$ & $+15.1\%$ \\
        Coord  & $1337.3 \pm 169.3$ & $2449.7 \pm 460.6$ & $2956.7 \pm 158.5$ & $3599.9 \pm 451.9$ & $151228.4 \pm 5796$ & $119.5 \pm 4.1$ & $4148.7 \pm 282.1$ & $+27.0\%$ \\
        \bottomrule
    \end{tabular}
\end{table*}

\subsection{Continuous control benchmarks}

As described in Section \ref{sec:method_ppo}, the PPO algorithm leverages an adaptive clipped objective function and usually more sophisticated optimizers such as ADAM \citep{kingma2014adam}. Will the variance reduction with more complicated control variates we showed in the previous subsection translate into better sample efficiency in practice? We investigate this question on the popular MuJoCo-based \citep{todorov2012mujoco} continuous control benchmarks provided by OpenAI Gym \citep{brockman2016openai}.

We covered 7 locomotion environments (Gym `-v2' versions): Ant, HalfCheetah, Hopper, Humanoid, HumanoidStandup, Swimmer, Walker2d. The Humanoid experiments are run for 3 million timesteps due to its complexity, while the other experiments are run for 1 million timesteps. Besides the different control variates, other components of PPO remain the same across different methods within each subplot. We tune $\lambda$ and $\rho$ for \emph{each} environment. We refer the readers to the supplementary material for the details, the time and space costs and the ablation study on the hyperparameters.

Fig.~\ref{fig:mujoco} and Tab.~\ref{tab:mujoco}, \ref{tab:mujoco100} demonstrate the advantages of the more complex forms of control variates: They yield faster learning speed particularly in hard environments such as Ant and Humanoid, although the difference in Tab.~\ref{tab:var} is seemingly small. Without much surprise, the fitted optimal scalar-valued baseline outperforms the vanilla value function baseline variant. We observe that the most elaborated coordinate-wise control variate actually leads to the highest sample efficiency, and both the fitted layer-wise and coordinate-wise control variates outperform the value or scalar control variates. The relative improvement over the value function CV is $\ge 27\%$ for coordinate-wise CV and $\ge 11\%$ for layer-wise CV measured in the normalized episodic returns.

On the disadvantageous side, the more complex control variates need larger effort to train, in terms of memory and computation cost. They are also prone to over-fitting the current trajectory sample at hand due to the higher degrees of freedom. 
However, we show it is possible to combat over-fitting and achieve good overall performance by utilizing the regularization introduced in Eq.~\ref{eq:l_baseline}.
Overall, we think the best choice of control variates depends on the trade-off between sample efficiency, computation budget, and hyper-parameter tuning effort. It can be worthwhile to explore the family of coordinate-wise control variates in policy gradient methods.

%% file: supp.tex
\section{Supplementary experiment details}

\subsection{Hyper-parameters for Sec. 4.2}

We list the hyper-parameters of the PPO algorithm that are shared across all methods in Tab.~\ref{tab:hyperparam}. Note that the Humanoid-v2 is harder than other environments, hence we set a few hyper-parameters differently (number of time-steps, parallel processes, learning rate, and PPO optimization parameters). The choices of some hyper-params that are unrelated to our proposal, such as clipping, GAE, coefficient, follow the original PPO paper.

\begin{table}[h]
    \caption{Hyper-parameters that are shared across methods.}
    \centering
    \begin{tabular}{lll}
        \toprule
        \bfseries Hyper-parameters & \bfseries Other env. & \bfseries Humanoid-v2  \\
        \midrule
        Total num. env steps & 1,000,000 & 3,000,000 \\
        Num. env steps per update & 2048 & 1024 \\
        Num. parallel env & 1 & 4 \\
        Discount factor $\gamma$ & 0.99 & 0.99 \\
        GAE-$\lambda$ & 0.95 & 0.95 \\
        Entropy coefficient & 0 & 0 \\
        Value loss coefficient & 0.5 & 0.5 \\
        Max grad norm & 0.5 & 0.5 \\
        Optimizer & ADAM & ADAM \\
        Learning rate (linearly annealed) & 0.0003 $\rightarrow$ 0 & 0.0001 $\rightarrow$ 0 \\
        Num. mini-batches per PPO epoch & 32 & 64 \\
        Num. PPO epochs per update & 10 & 10 \\
        PPO clip parameter & 0.2 & 0.2 \\
        Network architecture & \multicolumn{2}{l}{Linear[64] - Linear[64] - Linear[out]} \\
        Network activation & TanH & TanH \\
        \bottomrule
    \end{tabular}
    \label{tab:hyperparam}
\end{table}

In Tab.~\ref{tab:hyperparam_baseline}, we also list the per-environment hyper-parameters $(\lambda, \rho)$ accompanying the results in the main text Sec. 4, along with the environment observation and action space sizes. The number of policy net parameters (number of coordinates) depends on the sizes of these two spaces. We adopt a two hidden layer fully connected network architecture for all the policy, value and baseline functions. The hidden layer width is 64. Therefore the number of policy parameter is $\text{Obs. size} \times 64 + 64 + 64\times 64 + 64 + 64 \times \text{Act. size} + \text{Act. size} \times 2$, counting the weights, the biases and the action log-std parameters.

\begin{table}[h]
    \caption{Details about each environment and regularization hyper-parameters for the Sec.~4.2 results regarding fitting Scalar, Layer-wise and Coordinate-wise CV with $L_{\text{baseline}}$.}
    \centering
    \begin{tabular}{llllllll}
        \toprule
        & {Ant} & {HalfCheetah} & {Hopper} & {Humanoid} & {HumanoidStand} & {Swimmer} & {Walker2d}
        \\
        \midrule
        Obs. space     & 111   & 17   & 11   & 376   & 376   & 8    & 17 \\
        Action space   & 8     & 6    & 3    & 17    & 17    & 2    & 6 \\
        \#policy param & 11856 & 5708 & 5126 & 29410 & 29410 & 4868 & 5708 \\
        \midrule
        $\lambda$ & 0.01 & 0.01 & 0.01 & 0.1  & 0.0 & 0.0  & 0.01 \\
        $\rho$    & 0.1  & 0.1  & 0.05 & 0.01 & 0.0 & 0.01 & 0.0 \\
        \bottomrule
    \end{tabular}
    \label{tab:hyperparam_baseline}
\end{table}





\clearpage        
\section{Additional results}

\subsection{Mean squared error plot for Sec. 4.1}

Sec.~4.1 contains a table showing the variance estimates. Here we also draw the Mean Squared Error (MSE) plots of the gradient estimates, varying the size of the sample. The true policy gradient is approximated by a separate large sample of $3\times 10^5$ transitions and the learned value function baseline. Then the MSE is calculated against the approximated true gradient. As we can see, the MSE also becomes smaller with more advanced CVs, suggesting more accurate gradient estimation.

\begin{figure*}[h]
    \centering
    \includegraphics[width=0.5\textwidth]{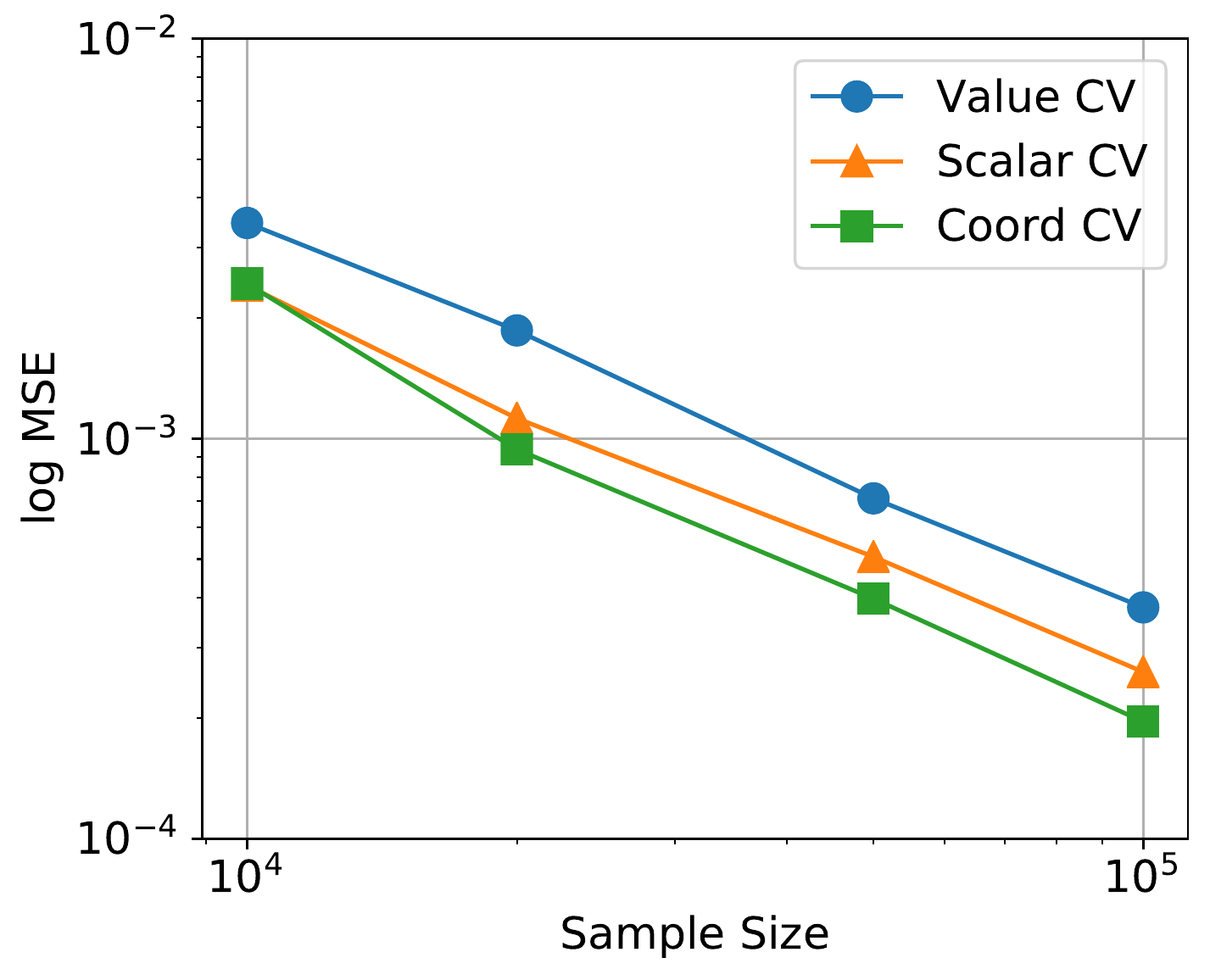}
    \caption{MSE of gradient estimators with different CVs.}
\end{figure*}

\subsection{Space and time comparison}

We take Humanoid-v2 as an example to investigate the space and time overhead of coordinate-wise CVs. Compared with Value CV, the more sophisticated Scalar, Layer, Coord CVs involve storing and training an extra neural network. The number of output units of this neural net depends on the CV used. Layer CV baseline has 7 outputs because there are 7 parameter tensors in the policy network. The number of outputs of Coord CV equals the number of coordinates of the policy network, which is determined by the state and action dimensions. We observe that the time overhead of Coord CV is manageable thanks to the efficient per-example AutoGrad implementation - only 0.6h extra training time compared with the Value CV (relative 0.6/2.19=27\% increase).

\begin{table}[h]
    \caption{Training space and time comparison of PPO with value, scalar, layer-wise and coordinate-wise CVs for the Humanoid-v2 environment, measured with an nVidia GTX 1080 GPU.}
    \centering
    \begin{tabular}{lllll}
        \toprule
        \bfseries{Humanoid}  & Value CV & Scalar CV & Layer CV & Coord CV  \\
        \midrule
        Num. baseline outputs & 1 & 1 & 7 & 29410 \\
        Frames per second & 380 &  309  &  305  & 299  \\
        Total training time & 2.19 hours & 2.67 hours  & 2.73 hours & 2.78 hours  \\
        \bottomrule
    \end{tabular}
    \label{tab:time}
\end{table}

\subsection{Effect of regularization parameter $\lambda$ and $\rho$}

We conduct ablation study on the regularization parameter $\lambda$ and $\rho$ in the baseline fitting loss. $\lambda$ interpolates between the fitting the minimum empirical variance baseline and the value function baseline, while $\rho$ tunes the proximal regularization strength.

\begin{table*}[h]
    \centering
    \caption{The effect of baseline fitting hyper-parameters $(\lambda, \rho)$ on performance. We report the mean episodic returns of \emph{all episodes} during training. All results are the average of 4 random seeds.}
    \small
    \begin{tabular}{l|l|lllll}
        \toprule
        \bfseries Ant   & $(0.01, 0.1)$ & $(0.01, 0.0)$ & $(0.01, 0.05)$ & $(0.0, 0.1)$ & $(0.1, 0.1)$  \\
        \midrule
        Scalar CV     & 210.25  & 179.88 & 189.66  & 194.40 & 146.54
        \\
        Layer-wise CV & 189.53  & 134.69 & 106.40  & 156.53 & 87.90 
        \\ 
        Coord-wise CV & 323.67  & 235.23 & 253.60  & 284.54 & 222.92
        \\
        \bottomrule
        \toprule
        \bfseries HalfCheetah   & $(0.01, 0.1)$ & $(0.01, 0.0)$ & $(0.01, 0.01)$ & $(0.0, 0.1)$ & $(0.1, 0.1)$  \\
        \midrule
        Scalar CV     & 1399.89  & 1554.35 & 1738.73  & 1308.23 & 1298.71
        \\
        Layer-wise CV & 1833.37  & 1808.32 & 1499.38  & 1311.79 & 1591.76
        \\ 
        Coord-wise CV & 1694.47  & 1337.72 & 1391.58  & 1332.84 & 1785.43
        \\
        \bottomrule
        \toprule
        \bfseries Hopper   & $(0.01, 0.05)$ & $(0.01, 0.01)$ & $(0.01, 0.1)$ & $(0.0, 0.01)$ & $(0.0, 0.1)$  \\
        \midrule
        Scalar CV     & 1160.30  & 1412.54 & 1094.03  & 1146.21 & 1183.94
        \\
        Layer-wise CV & 1395.81  & 1228.40 & 1290.87  & 1311.79 & 1365.25
        \\ 
        Coord-wise CV & 1408.58  & 1481.50 & 1403.77  & 1379.06 & 1495.21
        \\
        \bottomrule
        \toprule
        \bfseries Humanoid   & $(0.1, 0.01)$ & $(0.1, 0.05)$ & $(0.05, 0.0)$ & $(0.01, 0.01)$ & $(0.01, 0.05)$  \\
        \midrule
        Scalar CV     & 1034.93  & 1034.47 & 1094.03  & 1022.11 & 991.22
        \\
        Layer-wise CV & 902.80   & 824.81  & 885.59   & 885.32 & 848.48
        \\ 
        Coord-wise CV & 1056.05  & 991.96 & 1007.88  & 1119.19 & 1012.50
        \\
        \bottomrule
        \toprule
        \bfseries HumanoidStandup   & $(0.0, 0.0)$ & $(0.0, 0.01)$ & $(0.01, 0.0)$ & $(0.01, 0.01)$ & $(0.1, 0.01)$  \\
        \midrule
        Scalar CV     & 122945  & 126393 & 110376  & 122647 & 128471
        \\
        Layer-wise CV & 120230  & 107631 & 105924  & 111034 & 113181
        \\ 
        Coord-wise CV & 128267  & 100111 & 115444  & 110293 & 108881
        \\
        \bottomrule
        \toprule
        \bfseries Swimmer   & $(0.0, 0.01)$ & $(0.0, 0.0)$ & $(0.01, 0.0)$ & $(0.01, 0.01)$ & - \\
        \midrule
        Scalar CV     & 82.24  & 74.31 & 76.99  & 77.57 & -
        \\
        Layer-wise CV & 84.31  & 72.94 & 60.84  & 73.79 & -
        \\ 
        Coord-wise CV & 91.04  & 81.63 & 83.40  & 69.50 & -
        \\
        \bottomrule
        \toprule
        \bfseries Walker2d   & $(0.01, 0.0)$ & $(0.0, 0.0)$ & $(0.0, 0.01)$ & $(0.01, 0.01)$ & $(0.01, 0.1)$  \\
        \midrule
        Scalar CV     & 1337.64  & 1202.55 & 1403.35  & 1373.67 & 1363.48
        \\
        Layer-wise CV & 1325.74  & 958.75 & 1406.74  & 1428.29 & 1330.60
        \\ 
        Coord-wise CV & 1455.91  & 1440.88 & 1198.60  & 1301.73 & 1288.50
        \\
        \bottomrule
    \end{tabular}
    \label{tab:ablation}
\end{table*}

\section{Derivations}

\subsection{Derivation of Eq.~8}

Here is a more detailed derivation of the optimal state-dependent baseline. First, we know that minimizing the variance $\sV[\rvg^b]$ is equivalent to minimizing the expected gradient square term.
\begin{align*}
    \argmin_b \sV[\rvg^b] 
    &= \argmin_b \sum_{j=1}^d \E[{\rvg^b_j}^2] 
    \text{\qquad where} \\
    \sum_{j=1}^d \E[{\rvg^b_j}^2]
    \tag{Linearity of expectation} 
    &= \E[\sum_{j=1}^d {\rvg^b_j}^2] = \E \sum_{j=1}^d \big( \frac{\partial}{\partial \theta_j} \log \pi_\theta(\rva|\rvs) \big)^2
    \big( \hat{Q}_{\pi_\theta}(\rvs, \rva) - b(\rvs) \big)^2
    \\
    & = \E_\rvs \E_{\rva|\rvs} \left[ \| \nabla_{\theta} \log \pi_\theta(\rva | \rvs) \|^2
    \big( \hat{Q}_{\pi_\theta}(\rvs, \rva) - b(\rvs) \big)^2 \right]
    .
    \tag{Iterated expectation \& definition of $\nabla$}
\end{align*}

Consider the case where we can freely choose $b(\rvs)$ for every state, which implies $\min_b \E_\rvs \E_{\rva|\rvs}[\cdots] = \E_\rvs [ \min_{b(\rvs)} \E_{\rva|\rvs} [\cdots]]$. Then the optimal solution should satisfy the first order necessary condition:
\begin{align*}
    &0 = \frac{\partial}{\partial b(\rvs)} \E_{\rva|\rvs} \left[ \cdots \right]
    = 2 \E_{\rva|\rvs} \left[ \| \nabla_{\theta} \log \pi_\theta(\rva | \rvs) \|^2
    \big( b(\rvs) - \hat{Q}_{\pi_\theta}(\rvs, \rva) \big) \right]
    \\
    \implies\quad &\E_{\rva|\rvs} \left[ \| \nabla_{\theta} \log \pi_\theta(\rva | \rvs) \|^2 b(\rvs) \right]
    =
    \E_{\rva|\rvs} \left[ \| \nabla_{\theta} \log \pi_\theta(\rva | \rvs) \|^2 \hat{Q}_{\pi_\theta}(\rvs, \rva) \right]
    \\
    \implies\quad &b^*(\rvs) = \frac{\E_{\rva|\rvs \sim \pi_\theta} \left[
    \| \nabla_{\theta} \log \pi_\theta(\rva | \rvs) \|^2
    \hat{Q}_{\pi_\theta}(\rvs, \rva)
    \right] }
    {\E_{\rva|\rvs \sim \pi_\theta} \| \nabla_{\theta} \log \pi_\theta(\rva | \rvs) \|^2 }
    .
\end{align*}

\subsection{Derivation of Eq.~10}

The loss function for the baseline fitting is derived as follows. Similarly, we start by minimizing the variance, then add the regularization to make it better suitable for neural net fitting in practice.
\begin{align*}
    \argmin_{\phi} \sV[{\rvg^{c_\phi}}] =
    \argmin_{\phi} \E[\sum_{j=1}^d {\rvg^c_j}^2] 
    = \argmin_{\phi} \E \sum_{j=1}^d \big( \frac{\partial}{\partial \theta_j} \log \pi_\theta(\rva|\rvs) \big)^2
    \big( \hat{Q}_{\pi_\theta}(\rvs, \rva) - c_j(\rvs) \big)^2
    .
\end{align*}

The objective can be approximated by a Monte Carlo sampling version with $n$ state-action pairs,
\begin{align*}
    \min_{\phi}
    \frac1{nd} \sum_{i=1}^n \sum_{j=1}^d 
    \big( \frac{\partial}{\partial \theta_j} \log \pi_\theta(\rva_i|\rvs_i) \big)^2
    \big( \hat{Q}_{\pi_\theta}(\rvs_i, \rva_i) - c_{j,\phi}(\rvs_i) \big)^2 .
\end{align*}

Notice that this is similar to the usual value function fitting except for the $\frac{\partial}{\partial \theta_j} \log \pi_\theta(\rva_i|\rvs_i)$ weights. We therefore introduce the hyper-parameter $\lambda$ to interpolate between fitting an optimal baseline and fitting a value function. To normalize the numerical scale to match the usual value function learning, we divide the $\big( \frac{\partial}{\partial \theta_j} \log \pi_\theta(\rva_i|\rvs_i) \big)^2$ coefficient by its batch mean, such that
\begin{align*}
    \frac1{nd} & \sum_{i=1}^n \sum_{j=1}^d 
        \Big( (1-\lambda)
        \overline{ \big( \frac{\partial}{\partial \theta_j} \log \pi_\theta(\rva_i|\rvs_i) \big)^2 } + \lambda \Big) = 1.
\end{align*}

Furthermore, inspired by PPO and TRPO, we add a constraint to encourage conservative updates on the baseline:
\begin{align*}
    \min_\phi \frac1{nd} & \sum_{i=1}^n \sum_{j=1}^d 
        \Big( (1-\lambda)
        \overline{ \big( \frac{\partial}{\partial \theta_j} \log \pi_\theta(\rva_i|\rvs_i) \big)^2 }
        + \lambda \Big)
        \big( \hat{Q}_{\pi_\theta}(\rvs_i, \rva_i) - c_{j,\phi}(\rvs_i) \big)^2
    \\
    & \text{s.t.\quad}
    \sum_{i,j} \big( c_{j,\phi}(\rvs_i) - c_{j,\phi_{\text{old}}}(\rvs_i) \big)^2 \le \delta.
\end{align*}

The above problem can be turned into an unconstrained one with appropriate Lagrange multipliers ($\rho$). We pick a fixed $\rho$, resulting in the loss function Eq.~10 to be trained by a few steps of mini-batch gradient descent.